\documentclass{svproc}
\usepackage{amsmath}
\usepackage{amssymb}
\usepackage{graphicx} 
\usepackage{xcolor}
\usepackage{wrapfig}
\usepackage{floatrow}
\usepackage{mathtools}
\usepackage{subcaption}

\usepackage{amsfonts}
\usepackage[ruled, vlined, linesnumbered]{algorithm2e}
\usepackage{algorithmicx}
\usepackage[noend]{algpseudocode}
\usepackage[shortcuts]{extdash}
\usepackage[
raggedright
]{sidecap}   
\sidecaptionvpos{figure}{t}

\DeclareMathOperator{\diag}{diag}

\usepackage{url}

\begin{document}
\newcommand{\REVISION}[1]{#1}
\newcommand{\NEW}[1]{#1}
\urlstyle{tt}
\newcommand\blfootnote[1]{%
  \begingroup
  \renewcommand\thefootnote{}\footnote{#1}%
  \addtocounter{footnote}{-1}%
  \endgroup
}

\newtheorem{myproblem}{Problem}

\newcommand{\sref}[1]{Section~\ref{#1}}
\newcommand{\dref}[1]{Def.~\ref{#1}}
\newcommand{\aref}[1]{Alg.~\ref{#1}}
\newcommand{\apref}[1]{Appendix~\ref{#1}}

\def\dim{\ensuremath{D}}
\def\udim{\ensuremath{m}}
\def\cost{\ensuremath{\mathcal{J}}}
\def\ergmet{\ensuremath{\mathcal{E}}}
\def\fullState{\ensuremath{\bar{\mathbf{x}}}}
\def\regState{\ensuremath{x}} 
\def\regStateI{\ensuremath{x_i}} 
\def\augState{\ensuremath{z}} 
\def\augStateK{\ensuremath{z_k}}
\def\augStateMat{\ensuremath{\mathbf{Z}}}
\def\augSpace{\ensuremath{\mathcal{Z}}}
\def\Lam{\ensuremath{\mathbf{\Lambda}}}
\def\LamK{\ensuremath{\Lambda_k}}
\def\LamK{\ensuremath{\Lambda_k}}
\def\basisK{\ensuremath{F_k}}
\def\basisVec{\ensuremath{\mathbf{F}}}
\def\basisMat{\ensuremath{\mathbf{F}}}
\def\dBasisK{\ensuremath{DF_k}}
\def\dBasisMat{\ensuremath{D\mathbf{F}}}
\def\dynReg{\ensuremath{f}} 
\def\dynRegI{\ensuremath{f_\regStateI}} 
\def\dynAug{\ensuremath{f_{\augState}}} 
\def\dyn{\ensuremath{f}} 
\def\dynAugK{\ensuremath{\dot{\augStateK}}}
\def\solK{\ensuremath{w_k}}
\def\sol{\ensuremath{\mathbf{w}}}
\def\phiVec{\ensuremath{\mathbf{\Phi}}}
\def\vBasis{\ensuremath{\psi}}
\def\vBasisK{\ensuremath{\psi_k}}
\def\diag{\text{diag}}
\newcommand{\Jac}[1]{J(#1)}
\def\uPol{\ensuremath{\ctrl}}
\def\dPol{\ensuremath{\dist}}
\def\UPol{\ensuremath{\ctrl\od}}
\def\DPol{\ensuremath{\dist\od}}
\def\dynCtrl{\ensuremath{\mathbf{B}}}
\def\costCtrl{\ensuremath{\mathbf{R}}}

\def\V{\ensuremath{V}}
\def\Vg{\ensuremath{\V}}

\def\nBasisK{\ensuremath{\tilde{\basisK}}}
\def\zdot{\ensuremath{\dot{z}_k}}
\def\zint{\ensuremath{e_k}}
\newcommand{\Run}[2]{\smallint_{#1}^{#2}{\running(\trajf{\time},\uPol(\time))d\time}}
\def\epK{\ensuremath{\epsilon_k}}
\newcommand{\zInt}[2]{\smallint_{#1}^{#2}{\epK(\time)d\time}}
\newcommand{\zIntp}[2]{\left(\zInt{#1}{#2}\right)}
\newcommand{\zIntSet}[2]{\left\{\zInt{#1}{#2},\forall k\in\kSpace\right\}}
\newcommand{\ZInt}[2]{\Zint|_{#1}^{#2}}
\newcommand{\ZIntp}[2]{(\ZInt{#1}{#2})}
\def\Zint{\ensuremath{e}}
\def\dynE{\ensuremath{\mathbf{f_e}}}
\def\od{(\cdot)}
\def\dynEk{\ensuremath{f_e}}
\def\TO{\ensuremath{{t_0}}}
\def\Ti{\ensuremath{{t_1}}}
\def\Tii{\ensuremath{{t_2}}}
\newcommand{\sumk}[1]{\sum_{k\in\kSpace}{\left(#1\right)}}
\newcommand{\umint}[3]{\min_{\uPol_{[#1,#2]}}{{\max_{\dPol_{[#1,#2]}}}{\left\{#3\right\}}}}
\newcommand{\umintn}[2]{\min_{\uPol_{[#1,#2]}}{{\max_{\dPol_{[#1,#2]}}}}}

\def\od{(\cdot)}
\def\traj{\state}
\newcommand{\trajf}[1]{\state(#1)}
\def\Traj{\state\od}
\def\ztraj{\augState}
\newcommand{\ztrajf}[1]{\augState(#1)}
\def\zTraj{\augState\od}

\newcommand{\figref}[1]{Figure \ref{fig:#1}}
\newcommand{\Partial}[2]{\frac{\partial #1}{\partial #2}}

\def\This{RAnGE}
\newenvironment{problemSetup}
    {\noindent\makebox[\linewidth]{\rule{\columnwidth}{1pt}}}
    {\makebox[\linewidth]{\rule{\columnwidth}{0.5pt}}}

\def\info{\ensuremath{\phi}}
\def\RRR{\ensuremath{\mathbb{R}}}
\def\Re{\ensuremath{\mathbb{R}}}
\def\term{\ensuremath{h}}
\def\running{\ensuremath{g}}
\def\extrarunning{\ensuremath{\running_1}}
\def\ctg{\ensuremath{\V}}
\def\basisGen{\ensuremath{F}}
\def\state{\ensuremath{x}} 
\newcommand{\subtraj}[2]{\ensuremath{\traj_{[#1, #2]}}}
\def\ctrl{\ensuremath{u}} 
\def\ctrlel{\ensuremath{u}}
\def\ctrlSpace{\ensuremath{\mathcal{U}}}
\def\dist{\ensuremath{d}} 
\def\distel{\ensuremath{d}}
\def\distSpace{\ensuremath{\mathcal{D}}}
\def\TF{\ensuremath{{t_f}}}
\def\tRange{\ensuremath{[\TO, \TF]}}
\def\defeq{\coloneqq}
\def\zk{\ensuremath{z_k}}
\def\zs{\ensuremath{z}} 
\def\pos{\ensuremath{x_1}}
\def\vel{\ensuremath{x_2}}
\def\cost{\ensuremath{J}}
\def\pCost{\ensuremath{\hat{J}}}
\def\stateDim{\ensuremath{n}}
\def\ctrlDim{\ensuremath{n_\ctrl}}
\def\distDim{\ensuremath{n_\dist}}
\def\kmax{\ensuremath{N}}
\def\expDim{\ensuremath{v}}
\def\expSpace{\ensuremath{\mathcal{M}}}
\def\expMap{\ensuremath{M}} 
\def\expPt{\ensuremath{m}} 
\def\termWeight{\ensuremath{q_1}}
\def\runningWeight{\ensuremath{q}}
\def\time{\ensuremath{t}}
\def\dummytime{\ensuremath{\tau}}
\def\ck{\ensuremath{c_k}}
\def\dyn{\ensuremath{f}}
\def\barr{\ensuremath{\text{barr}}}
\def\params{\ensuremath{\theta}}
\def\DNN{\ensuremath{V_{\params}}}
\def\Loss{\ensuremath{\ell}}
\def\BoundLoss{\ensuremath{\Loss_{B}}}
\def\AvBoundLoss{\ensuremath{\Bar{\BoundLoss}}}
\def\DiffLoss{\ensuremath{\Loss_{D}}}
\def\AvDiffLoss{\ensuremath{\Bar{\DiffLoss}}}
\def\lossLambda{\ensuremath{\lambda}}
\def\sBuff{\ensuremath{b}}
\def\TrivSolName{trivial "solution"}
\def\trivsol{\ensuremath{{\ctg}_{triv}}}
\def\trivctg{\ensuremath{\ctg'}}
\def\Vinter{\ensuremath{{\ctg}_{inter}}}
\def\smalltime{\ensuremath{\epsilon}}
\def\tDistExp{\ensuremath{k_1}}
\def\tExp{\ensuremath{k_2}}
\def\pHam{\ensuremath{H}}
\def\altErg{\ensuremath{\Bar{\Erg}}}

\newcommand{\vmag}[1]{\lVert #1 \rVert}
\def\Erg{\ensuremath{\mathcal{E}}}
\def\xa{\ensuremath{\Bar{\state}}} 
\def\fa{\ensuremath{\Bar{\dyn}}}
\def\tp{^\intercal} 
\def\es{\ensuremath{\mathcal{X}}} 
\def\kSpace{\ensuremath{\mathcal{K}}} 
\def\esa{\ensuremath{\overline{\es}}} 
\def\Nat{\{0,1,...,N-1\}}
\def\Rpos{\RRR_{\geq 0}}

\def\uDisc{\mathbf{\ctrl}}
\def\dDisc{\mathbf{\dist}}
\def\xDisc{\mathbf{\state}}
\def\zDisc{\mathbf{\augState}}
\def\dt{\Delta \time}
\def\N{T}
\def\niter{$N$}
\mainmatter              
\title{
RAnGE: Reachability Analysis\newline for Guaranteed Ergodicity}
\titlerunning{Reachability for Ergodicity}  
%
\author{Henry Berger \and Ian Abraham}
\authorrunning{Berger and Abraham} 
%
%
\institute{Department of Mechanical Engineering\\Department of Computer Science\\
Yale University, New Haven CT 06511, USA\\
\email{\{henry.berger,ian.abraham\}@yale.edu}}
\maketitle              
\begin{abstract}
    This paper investigates performance guarantees on coverage-based ergodic exploration methods in environments containing disturbances. 
    Ergodic exploration methods generate trajectories for autonomous robots such that time spent in each area of the exploration space is proportional to the utility of exploring in the area. 
    We find that it is possible to use techniques from reachability analysis to solve for optimal controllers that guarantee ergodic coverage and are robust against disturbances.
    We formulate ergodic search as a differential game between the controller optimizing for ergodicity and an external disturbance, and 
    we derive the reachability equations for ergodic search using an extended-state Bolza-form transform of the ergodic problem.
    Contributions include the computation of a continuous value function for the ergodic exploration problem and the derivation of a controller that provides guarantees for coverage under disturbances. 
    Our approach leverages neural-network-based methods to solve the reachability equations; we also construct a robust model-predictive controller for comparison.
    Simulated and experimental results demonstrate the efficacy of our approach for generating robust ergodic trajectories for search and exploration on a 1D  system with an external disturbance force.
    
\end{abstract}

\section{Introduction}

    \par Autonomous 
    mobile robots have been used for a number of search and exploration applications, including search and rescue \cite{tomic_urban_uav_sar,schedl_forest_uav_sar}, mapping~\cite{nuske_river_mapping,qin_mapping}, and extraterrestrial exploration~\cite{serna_mars}. One challenge is that many real-world applications of autonomous search occur in areas with hazards that impede search efforts. Hazards for autonomous mobile robots can be broadly categorized into two classes: static obstacles, such as walls or objects, which do not vary with time; and dynamic disturbances, such as wind and moving obstacles, which can vary over time. It is advantageous to find control solutions that provide guarantees on effective coverage, even in the presence of hazards.
    \blfootnote{For our code, see \url{https://github.com/ialab-yale/RAnGE}.}
    \par Recent advances in ergodic coverage-based exploration algorithms have demonstrated their efficacy for generating trajectories that provide coverage of a search area \cite{dressel_optimality_2018}. Ergodic exploration methods optimize trajectories such that over time, a robot's trajectory spends time in each area of a search space proportional to the information to be gained in the area \cite{mathew_metrics_2011}, providing sparser coverage in low-information areas and more meticulous coverage in high-information areas.
    However, current ergodic controllers have no formal performance guarantees in the presence of disturbances. 
    Therefore, this paper investigates whether it is possible to provide performance guarantees on coverage-based ergodic control methods under dynamic disturbances. 
    
    \begin{figure}[t]
         \includegraphics[width=\linewidth]{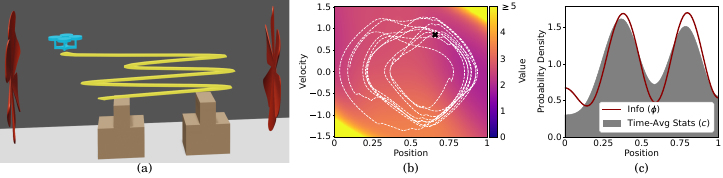}
            \caption{\textbf{Overview of \This.} An agent explores an area with distributed information, subject to a disturbance (a), using a precomputed value function (b) to achieve comprehensive ergodic coverage (c). Note that in (b), higher values denote less desirable states.}
            \label{fig:Blender Render}
            \hfill
    \end{figure}

    \par Methods in reachability analysis offer a path toward providing such performance guarantees on ergodic control problems. Reachability analysis is a method for computing control policies with performance guarantees in wors\REVISION{t}-case disturbance scenarios~\cite{lygeros_reachability_2004,bansal_hamilton-jacobi_2017,chen_decomposition_2018}. 
    Unfortunately, the prerequisite problem structure for using reachability analysis is not present in the canonical ergodic control problem formulation.
    In this paper, we investigate suitable formulations to construct reachability-based controllers for ergodic search that guarantee ergodic coverage of an area in the presence of dynamic external disturbances. 
    
    \par Our approach structures the ergodic control problem as a differential game between control and external disturbance force, and we employ a Bolza-form augmented-state transform of the ergodic metric that has also been used in other ergodicity-related work~\cite{mathew_metrics_2011,de_la_torre_ergodic_2016,dong2023}.
    We show that the problem transform is equivalent to the original ergodic control problem and demonstrate compatibility with existing Hamilton-Jacobi-Isaacs (HJI) reachability methods~\cite{bansal_hamilton-jacobi_2017}. 
    We take advantage of computational advancements in HJI methods~\cite{bansal_deepreach_2021} to solve controllers capable of generating robust ergodic trajectories, although the limitations of these computational methods restrict the complexity of problems for which our controller can currently be computed.
    In summary, our main contributions are as follows: 
    
    \begin{itemize}
        \item Derivation of Hamilton-Jacobi-Isaacs reachability conditions for the problem of ergodic exploration;
        \item Computation of a robust ergodic controller under dynamic disturbances; 
        \item Demonstration of the effectiveness of robust reachability-based ergodic controllers for 1D coverage on a double-integrator system (illustrated in \figref{Blender Render}); and
        \item Derivation of a robust ergodic model-predictive controller (reMPC) for ergodic exploration in the presence of disturbance.
    \end{itemize}
    
    \par This paper is structured as follows: \sref{Related Work} outlines the related work, \sref{Preliminaries} provides an overview of preliminary background information on ergodic exploration and reachability analysis, \sref{Main} derives the main contribution of this work, \sref{Results} demonstrates performance results of the proposed method, and \sref{Conclusion} provides a discussion and conclusion. The appendices provide implementation details, proofs, and an additional derivation.
    
\section{Related Work} \label{Related Work}
\subsection{Reachability Analysis}

    \par Robots operating in the real world are often subject to disturbances, so it is desirable to have guarantees of performance for robotic systems in worst-case conditions. Reachability analysis is a verification technique for robotic systems under such performance scenarios. \REVISION{For a bounded, time-varying disturbance,} these analysis methods can determine the worst-case disturbance and provide a lower bound on performance in all possible conditions. Reachability has already been used for a variety of autonomous robotics applications, including control and obstacle avoidance for drones \cite{chen_decomposition_2018,shao_reachability-based_2021}, collision avoidance for airplanes \cite{bansal_deepreach_2021,julian_guaranteeing_2019}, and safe navigation for autonomous vehicles~\cite{bansal_deepreach_2021,shao_reachability-based_2021}.
    
    \par Reachability analysis requires finding the viscosity solution to the Hamilton-Jacobi-Isaacs (HJI) partial differential equation \cite{lygeros_reachability_2004,evans_differential_1983}; for most applications, this cannot be done analytically and must be done numerically. \REVISION{For any complex problem where an analytical solution is not feasible, the numerical solution is only an approximation of the value function. In some cases, it is possible to guarantee that all inaccuracies are conservative, which maintains the performance guarantees~\cite{chen_decomposition_2018,jiang2017usingneuralnetworkscompute}. In other cases, the reachability-based controllers cannot quite provide the formal guarantees promised by reachability analysis, but as the accuracy of the numerical solution improves, a violation of the guarantees becomes less likely.}
    
    \par Much research on reachability analysis has focused on low\-/dimensional systems. \REVISION{Additionally, due to the difficulty of learning the value function, reachability is often used modularly to provide safety guarantees (e.g., collision avoidance), without providing performance guarantees relating to the robot's main objective~\cite{herbert_fastrack_2017,majumdar_funnels}.
    In this work, we generate approximate solutions to the HJI problem for the case of ergodic exploration, a high-dimensional problem where reachability is used to guarantee performance, not merely safety.}
    \paragraph{The Curse of Dimensionality.} The HJI equation is often solved through discretizing the state space on a grid. However, the computational cost of this approach scales exponentially with the problem state dimensionality, and as a result, grid-based solutions to the HJI equation are limited to low\-/dimensional problems. State decomposition methods \cite{chen_decomposition_2018,he_decomposition} can mitigate this scaling, but many high-dimensional problems are still difficult to solve on a grid. Furthermore, decomposition methods compute individual level sets of the value function, not a complete, differentiable value function.
    \par Recent advances using neural networks for solving differential equations \cite{sitzmann_implicit_2020} have provided new methods for solving the reachability equations \cite{jiang2017usingneuralnetworkscompute,nakamura2021,bansal_deepreach_2021}. For a neural network, the computational time required to find a satisfactory solution depends on the complexity of the solution, rather than strictly on the dimensionality of the state space. Higher-dimensional problems will often have more complex value functions and therefore take longer to learn, but the scaling is better than exponential~\cite{bansal_deepreach_2021}.
    Neural-network-based solvers have demonstrated solutions of the reachability equations for 10D problems with disturbance~\cite{bansal_deepreach_2021} and 30D problems without disturbance~\cite{nakamura2021}.

\subsection{Ergodic Exploration}

    \par In this paper, we apply reachability to a quintessentially trajectory-based optimization problem: coverage-based exploration. In general, the goal of exploration is to maximize coverage of a search area, and for problems involving a single mobile robot, coverage is inherently a property of an entire trajectory, not of an instantaneous state.
    
    \par For discretized exploration problems, a variety of methods already exist to provide formal guarantees on performance. In discrete time and space, exploration can be represented as a traveling salesperson problem, enabling the application of fast, locally optimal algorithms~\cite{lin_tsp_approximate} or slower, formally guaranteed globally optimal algorithms~\cite{miller_tsp_exact}. Voronoi methods can be used to guarantee maximal coverage of a continuous exploration area by a set of agent locations~\cite{cortes_voronoi_coverage,chen_voronoi}; these locations can be treated as the static locations of a swarm of agents or as the discretized trajectory of a single agent, but Voronoi methods cannot generate continuous trajectories. The motions of most robotic systems are continuous in time and space, so discrete-time and discrete-space methods necessarily introduce simplifications, which can lead to non-optimal results. 
    
    \par Ergodic exploration is a commonly-used technique for continuous\hspace{0pt}-time, continuous\hspace{0pt}-space exploration. Ergodic controllers seek to minimize the difference between the (expected) spatial distribution of information and the time-averaged spatial distribution of an agent's trajectory \cite{mathew_metrics_2011}. Dressel and Kochenderfer \cite{dressel_optimality_2018} proved that ergodic exploration is an optimal strategy for gathering information if, as is often the case in real-world applications, the amount of information in a location is depleted as the agent spends time in that location.

    \par Ergodic exploration has been combined with formally guaranteed static obstacle avoidance, using methods such as vector fields~\cite{salman_multi-agent_2017} and control barrier functions~\cite{lerch2023}. In this work, we introduce formal performance guarantees with time-varying disturbances into ergodic exploration, broadening the set of situations in which ergodic search can safely and effectively be used.

\section{Problem Setup and Definitions} \label{Preliminaries}

    \par We consider a general system with state 
    $\state\in\es\subseteq\RRR^\stateDim$,
    control input $\ctrl\in\ctrlSpace\subseteq\RRR^{\ctrlDim}$,
    disturbance input $\dist\in\distSpace\subseteq\RRR^{\distDim}$, and (possibly nonlinear) dynamics
    \begin{equation}
        \dot{\state}(\time)=
        \dynReg(\state(\time),\ctrl(\time),\dist(\time)).
        \label{Disturbance Dynamics}
    \end{equation}
    In this work, we consider the system over a finite time interval $t \in [\TO, \TF]$. The control input is generated by a control policy $\UPol:\es\times\tRange\to\ctrlSpace$, also called the \textit{controller}; the disturbance input is generated by a disturbance policy $\DPol:\es\times\tRange\times\ctrlSpace\to\distSpace$, also called the \textit{disturbance}.

    \par We assume that the disturbance has an instantaneous information advantage,  meaning that the disturbance can respond instantaneously to the controller, although it cannot anticipate the controller's future actions \cite{bansal_hamilton-jacobi_2017}. More concretely, at any time \time, the control is a function of \state\ and \time\ only, whereas the disturbance can be a function of \state, \time, and \ctrl.

\subsection{Ergodic Exploration}

    \par Ergodicity is a measure of the effectiveness of a trajectory in covering distributed information in a known bounded space. 
    We consider exploration of a \expDim-dimensional space $
    \expSpace=[0,L_0]\times[0,L_1]\times \cdots [0,L_{v-1}]\subset \RRR^\expDim$, where $L_i$ are bounds on the exploration space. We define a function $\expMap :\es\to\expSpace$ that maps state $\state\in\es$  to a point in the exploration space $\expPt\in\expSpace$.\footnote{\REVISION{Often, the exploration task only involves a subset of the agent's state (e.g., pose), though the state includes additional variables (e.g., velocity). The mapping \expMap\ between \es\ and \expSpace\ accounts for this distinction; for example, $\expMap$ could be a selection matrix that selects only the pose components of the state.}}
    \begin{definition}[Ergodicity]
    Given a function $\info:\expSpace\to\Re_{>0}$, an ergodic trajectory $\Traj:\tRange\to\es$ is one that spends time in each region of \expSpace\ in proportion to the integral of \info\ over that region, such that
    \begin{equation}
        \lim_{\time\to\infty}{\frac{1}{\time-\TO}\int_{\TO}^{\time}{\basisGen(\expMap\circ\trajf{\dummytime})d\dummytime}}=\int_{\expSpace}{\basisGen(\expPt)\info(\expPt)d\expPt}
    \end{equation}
    for all Lebesgue integrable functions $\basisGen:\expSpace\to\RRR$\emph{~\cite{mathew_metrics_2011}}. 
    \end{definition}

    \par  \REVISION{The definition of \info\ is problem-specific. For search and exploration applications, \info\ is an information density function that represents the probability of discovering information at a location.}
    Following \cite{mathew_metrics_2011}, ergodicity is often measured by a comparison in the Fourier domain between the information density $\info$ and the time-averaged trajectory statistics, which are defined as follows.
    
    \begin{definition}[Time-averaged trajectory statistics]
        The time-averaged location statistics of the trajectory are given by 
        \begin{equation}
            c(\Traj,\time) \defeq \frac{1}{\time-\TO}\int_{\TO}^\time\delta[\expMap(\trajf{\dummytime})]d\dummytime,
        \end{equation}
        where $\delta$ is the Dirac delta function.
    \end{definition}
    
    In order to measure whether a trajectory is ergodic, we define a metric on ergodicity. 
    
    \begin{definition}[Spectral ergodic metric]
        Following \emph{\cite{mathew_metrics_2011}}, the ergodic metric is given by
        \begin{equation}
            \Erg(\Traj, \time)\defeq\sum_{k\in\kSpace}{\Lambda_k\left(c_k(\Traj,\time)-\phi_k\right)^{2}},
            \label{Ergodic Metric Definition}
        \end{equation}
        where the time-average trajectory statistics $c$ and the information density function $\info$ are encoded in the coefficients $c_k$ and $\phi_k$, respectively, calculated using the Fourier transform
        \begin{equation}
            c_k(\Traj,\time)\defeq\frac{1}{\time-\TO}
            \int_{\TO}^{\time}{\basisK(\expMap\circ\trajf{\dummytime})d\dummytime}
            ,\ \ 
            \phi_k\defeq
            \int_\expSpace{\basisK\left(\expPt\right)\info(\expPt)d\expPt},
            \label{Fourier Coefficients}
        \end{equation}
        for modes $k\in\kSpace=\{0,1,2,...,\kmax-1\}^\expDim$.
        Here, the set $\{\Lambda_k\in \mathbb{R}\}$ is a sequence of weighting coefficients defined for each $k\in\kSpace$ as $\Lambda_k\defeq(1+\vmag{k}_2^2)^{-s}$, where $s\defeq(v+1)/2$.
    \end{definition}
    We use the cosine basis function 
        $
        \basisK(\expPt) \defeq \frac{1}{h_k} \prod_{i=0}^{\expDim-1}{\cos{\left(\frac{k_i\pi\expPt_i}{L_i}\right)}},
    $ where $\frac{1}{h_k}$ is a normalizing constant~\cite{mathew_metrics_2011}, but other orthonormal bases would also suffice.
    The ergodic metric measures how far a trajectory is from ergodicity.
    The canonical ergodic trajectory-optimization problem is the minimization of the ergodic metric with respect to a given information metric \info~\cite{mathew_metrics_2011,scott_capturing_2009} (see Problem \ref{def:Disturbed Ergodic} below).
        
\subsection{Reachability Analysis}

    \par Reachability analysis is a technique for solving differential games via the Hamilton-Jacobi-Isaacs (HJI) partial differential equation, which can be used to model safety requirements in control systems. \NEW{We first consider a problem in which,} subject to dynamic constraints, the controller and disturbance seek to minimize and maximize, respectively, a cost function with Bolza form (Problem \ref{def:bolza}).

    \begin{myproblem}[Bolza-form optimal control with disturbance] \label{def:bolza}
        Given a terminal cost $\term:\es\to\Rpos$ and a running cost $\running:\es\times\ctrlSpace\to\Rpos$, for some choice of $\state_0$ and $\TO$, find
        \begin{subequations}
            \label{Bolza Form}
            \begin{equation}
                \uPol^*\od\defeq\arg \min_{\UPol}\max_{\DPol}\left\{\cost(\state_0, \TO, \UPol, \DPol)\right\},
                \label{eq:u_opt}
            \end{equation}
            where the cost function $\cost$ is defined as
            \begin{equation}
                \cost\left(\state_0, \TO,\UPol, \DPol\right)\defeq\term(\trajf{\TF}) 
                +\int_{\TO}^{\TF}{\running(\trajf{\dummytime},\uPol(\dummytime))d\dummytime},
                \label{Bolza Cost}
            \end{equation}
            subject to the following conditions ($\forall\time\in\tRange$, where applicable):
                \begin{align}
                   \ctrl(\time)&=\uPol(\trajf{\time}, \time)
                    \\
                   \dist(\time)&=\dPol(\trajf{\time}, \time, \ctrl(\time))
                    \\
                    \dot{\state}(\time)&=\dynReg(\state(\time), \ctrl(\time), \dist(\time))
                    \\
                    \state(\TO)&=\state_0.
                \end{align}
        \end{subequations}
    \end{myproblem}

    \begin{definition}[Value function] \label{def:cost-to-go}
    The value function $\ctg:\es\times\tRange \to \RRR$, also called the optimal cost-to-go function, is the cost of the remainder of the trajectory, starting at the given time and state, ending at time $\TF$, \REVISION{with optimal control and disturbance.}
        
    \begin{equation}
        \ctg(\state,\time)=\min_{\UPol}\max_{\DPol}\left\{\cost(\state,\time, \UPol, \DPol\right)\}.
        \label{Cost-to-go}
    \end{equation}
    \end{definition}
    \par As proved in Theorem 4.1 of \cite{evans_differential_1983}, the \REVISION{value function} for a Bolza-form optimization problem (Problem \ref{def:bolza}) is a viscosity solution to the Hamilton-Jacobi-Isaacs optimality conditions \cite{evans_differential_1983}.

    \begin{definition}[Hamilton-Jacobi-Isaacs optimality conditions]
        \begin{subequations}
            For the Bolza-form optimization problem (Problem \ref{def:bolza}), the value function $\V:\es\times\tRange\to\RRR$ satisfies the conditions
            \noindent
            \label{HJI Problem}
            \begin{equation}
                -\Partial{\ctg}{t}= \pHam(\state,\time),\ \forall t\in[\TO,\TF],\forall\state\in\es
                \label{HJI Diffeq}
            \end{equation}
            \begin{equation}
                \ctg(\state, \TF)=\term(\state),\ \forall \state\in\es,
                \label{Terminal}
            \end{equation}
            where the Hamiltonian $\pHam$ is defined as
            \begin{equation}
                \pHam(\state,\time) \defeq 
                \min_{\ctrl\in\ctrlSpace}\max_{\dist\in\distSpace}{\Bigr{\{} 
                g(\state, \ctrl) + \left(\Partial{\ctg(\state,\time)}{\state}\right)\tp \dynReg(\state,\ctrl,\dist)
                \Bigl{\}}}.
                        \label{Hamiltonian}
            \end{equation}
        \end{subequations}
    \end{definition}
    
    \par Once value function \ctg\ is known, the HJI equation \eqref{HJI Diffeq} and the definition of the Hamiltonian \eqref{Hamiltonian} entail that the optimal controller $\uPol^*\od$ and worst-case disturbance $\dPol^*\od$ can be calculated directly as
    \begin{subequations}
        \begin{align}        
            \label{Optimal Control}
            \uPol^*(\state,\time) =\arg \min_{\ctrl\in\ctrlSpace}\bigg\{&\max_{\dist\in\distSpace}{\Bigr{\{}
            \hat{H}(\state, \time, \ctrl, \dist)
            \Bigl{\}} }\bigg\} \\
            \label{Optimal Disturbance}
            \dPol^*(\state,\time, \ctrl) =\arg &\max_{\dist\in\distSpace}{\Bigr{\{}
            \hat{H}(\state, \time, \ctrl, \dist)
            \Bigl{\}} },
        \end{align}
    \label{Optimals}
    \end{subequations}where $\hat{H}$ represents the target of the minimax condition in the Hamiltonian $\pHam$ \cite{isaacs1965differential}.
    This optimal controller is one of the main results of reachability analysis, since the controller can be used as a robust policy against worst-case disturbance. We seek to find such an optimal controller for exploration problems, in order to provide sufficient coverage guarantees for exploration with disturbance.

	
\section{Reachability Analysis for Guaranteed Ergodic Exploration} \label{Main}

    We now consider the problem of ergodic control subject to disturbance, as formulated in Problem \ref{def:Disturbed Ergodic}.
        \begin{myproblem}[Ergodic control with disturbance] \label{def:Disturbed Ergodic}
        
            Solve an optimization problem of the same form as the Bolza optimization problem with disturbance (Problem \ref{def:bolza}), but with the modified cost function
        \begin{equation}
            \cost\left(\state_0, \TO, \UPol, \DPol\right)\defeq\Erg(\Traj, \TF)+\term(\trajf{\TF})
            +\int_{\TO}^{\TF}{\running(\trajf{\dummytime},\uPol(\dummytime))d\dummytime}.
            \label{Disturbed Ergodic}
        \end{equation}
    \end{myproblem}

This problem is not strictly compatible with Bolza form, because the ergodic cost depends on the entire trajectory $\Traj$, not just on the terminal state $\trajf{\TF}$. Consequently, we reformulate the ergodic control problem in order to apply reachability analysis. 
    To do so, we first derive an equivalent form of the ergodic metric via an \emph{augmented-state ergodic} variable. 

    \begin{lemma}[Augmented-state ergodic metric formulation~\cite{mathew_metrics_2011,de_la_torre_ergodic_2016}]
        The ergodic metric defined in \eqref{Ergodic Metric Definition} can be equivalently defined as
        \label{metric_lemma}
        \begin{align}
            \Erg(\Traj,\TF)
            &=\sumk{\Lambda_k\left(c_k(\Traj,\time)-\phi_k\right)^{2}}\nonumber \\ 
            &=\frac{1}{t_f-t_0} \sumk{\Lambda_k \augState_k(\TF)^2}
            \label{Equivalent Metrics}
        \end{align} 
         where $z_k(\cdot)$ is the solution to the differential equation 
        \begin{subequations}
            \label{eq:dynAug}
            \begin{align}
                \dot{\augStateK}(\time)&=\basisK(\expMap\circ\trajf{\time})-\phi_k \\
                \augStateK(\TO)&=0\ \forall k\in\kSpace.
            \end{align}
        \end{subequations}

    \end{lemma}
    \begin{proof} 
        See \NEW{\apref{metric_proof}}. \qed
    \end{proof}

    This formulation of the ergodic metric \eqref{Equivalent Metrics} can be used to rewrite the ergodic control problem with disturbance (Problem \ref{def:Disturbed Ergodic})  in Bolza form.

    \begin{myproblem}[Bolza-form ergodic control with disturbance] \label{prob:Ergodic Bolza}
        Given a terminal cost $\term:\es\to\Rpos$ and a running cost $\running:\es\times\ctrlSpace\to\Rpos$, for some choice of $\state_0$ and \TO, find
        \begin{subequations}
            \label{Ergodic Bolza Form}
            \begin{equation}
                \uPol^*\od\defeq \arg \min_{\UPol}\max_{\DPol}\left\{\cost(\state_0, \TO, \UPol, \DPol)\right\},
                \label{Ergodic Bolza Minimax}
            \end{equation}
            where the cost function $\cost$ is defined as
            \begin{align}
                \cost\left(\TO,\state_0, \UPol, \DPol\right)=&
                \frac{1}{t_f-t_0} \sumk{\Lambda_k \augStateK(\TF)^2}+\term(\trajf{\TF}) \nonumber \\
                &+\int_{\TO}^{\TF}{\running(\trajf{\dummytime}, \uPol(\dummytime))d\dummytime},
                \label{Ergodic Bolza Cost}
            \end{align}
            subject to the following conditions ($\forall t\in\tRange$ and $\forall k\in\kSpace$, where applicable):
                \begin{alignat}{2}
                    \label{eq:first_traj_cond}
                   \ctrl(\time)&=\uPol(\state(\time),\augState(\time), \time)
                    ,
                   &&\dist(\time)=\dPol(\state(\time), \augState(\time), \time, \ctrl(\time))
                    ,
                    \\
                    \dot{\state}(\time)&=\dynReg(\state(\time), \ctrl(\time), \dist(\time)),
                    &&\state(\TO)=\state_0, \\
                    \dot{\augState}_k(\time)&=\basisK(\expMap\circ\trajf{\time})-\phi_k,
                    \hspace{0.25in}&&\augStateK(\TO)=0.
                    \label{eq:last_traj_cond}
                \end{alignat}
        \end{subequations}
        Note that $\augState\in\augSpace= \RRR^{N^\expDim}$ is a vector of $\augStateK, \forall k\in\kSpace$.
    \end{myproblem}

    \par Problem \ref{prob:Ergodic Bolza} is a canonical Bolza-form problem (Problem \ref{def:bolza}), which enables the application of reachability analysis. 

    \begin{theorem}[Ergodic Hamilton-Jacobi-Isaac PDE] \label{prob:Ergodic HJI}
     \label{thm:Guarantee}
         The solution to the ergodic control problem with disturbance (Problem \ref{prob:Ergodic Bolza}) is given by the value function $\ctg:\es\times\augSpace\times\tRange\to\RRR$, optimal controller $\uPol^*\od:\es\times\augSpace\times\tRange\to\ctrlSpace$, and optimal disturbance $\dPol^*\od:\es\times\augSpace\times\tRange\times\ctrlSpace\to\distSpace$ such that $\ctg$ satisfies
        \begin{subequations}
            \label{HJI Ergodic Problem}
                \begin{align}
                \label{Ergodic HJI Diff}
                -\Partial{\V(\state,\augState,\time)}{\time}&=\min_{\ctrl\in\ctrlSpace}\max_{\dist\in\distSpace}\left\{\hat{\pHam}(\state,\augState,\time, \ctrl,\dist)\right\}, \text{where}\\
                \hat{\pHam}(\state,\augState,\time, \ctrl,\dist)&=\bigg\{
                \running(\regState, \ctrl)+
                \left(\Partial{\V(\state,\augState,\time)}{\regState}\right)\tp
                \dynReg(\regState,\ctrl,\dist)\nonumber \\
                &\phantom{=\bigg\{}+\sumk{\left(\Partial{\V(\state,\augState,\time)}{\augStateK}\right)\tp\left(\basisK(\expMap(\state))-\phi_k
                \right)} 
                \bigg\},
                \label{Ergodic HJI Hamiltonian}
                \end{align}
            with terminal condition
            \begin{equation}
                \ctg(\regState,\augState,\TF)=\frac{1}{\TF-\TO} \sumk{\Lambda_k \augStateK(\TF)^2}+\term(\regState),
                \label{Ergodic HJI Terminal}
            \end{equation}
        \end{subequations}
        and such that $\uPol^*\od$ and $\dPol^*\od$ satisfy
        \begin{subequations}
                \begin{align}
                \uPol^*(\state,\augState,\time)&=\arg\min_{\ctrl\in\ctrlSpace}\max_{\dist\in\distSpace}\left\{\hat{\pHam}(\state,\augState,\time, \ctrl,\dist)\right\},\label{eq:range_control}\\
                \dPol^*(\state,\augState,\time,\ctrl)&=\arg\max_{\dist\in\distSpace}\left\{\hat{\pHam}(\state,\augState,\time, \ctrl,\dist)\right\}.\label{eq:range_disturbance}
            \end{align}
            \label{eq:range_ctrl_dist}
        \end{subequations}
    \end{theorem}
    \begin{proof}
        (Informal) 
        The proof follows \cite{galperin2008} and Section 4.2 of \cite{isaacs1965differential}. Using the definition of the value function (\dref{def:cost-to-go}) with the ergodic cost function \eqref{Ergodic Bolza Cost}, the value function at the terminal time can be shown to satisfy \eqref{Ergodic HJI Terminal}. We then break the value function into Isaacs' canonical principle of optimality (Eq. 2.7 of~\cite{galperin2008}), which is a minimax variant of the Bellman principle of optimality. From there, following the standard reachability derivation~\cite{isaacs1965differential,galperin2008} yields the remaining conditions for $\ctg,\uPol^*,\dPol^*$.
        A more detailed proof is provided in \NEW{\apref{pr:Guarantee}}. \qed
    \end{proof}
        \NEW{
        We additionally prove in \apref{sec:Derivation} that the canonical and extended-state formulations of the ergodic control problem (Problems \ref{def:Disturbed Ergodic} and \ref{prob:Ergodic Bolza}, respectively) yield the same HJI equation \eqref{HJI Ergodic Problem}. This demonstrates that the presence of extended states in the HJI equation \eqref{HJI Ergodic Problem} is a feature of the problem structure, not an artifact of the choice of problem formulation.
        }
        
        
    \par The optimal controller $\uPol^*$ from Theorem \ref{thm:Guarantee} is robust in the sense that it satisfies the following lemmas, which are direct consequences of \eqref{Ergodic Bolza Minimax}. 
    \begin{remark}[Cost guarantees]
    Applying $\uPol^*\od$ starting at $(\state,\augState,\TO)$ guarantees that despite any possible disturbance $d\in\distSpace$, the trajectory cost will be upper-bounded by  $\ctg(\state,\augState,\TO)$; furthermore, no controller can guarantee an upper bound lower than $\ctg(\state,\augState,\TO)$. 
    \end{remark}
    \begin{remark}[Reachable coverage set] 
    If $\running(\state,\time)=0$ and $\term(\state)=0$, then the set $\{(\state,\augState)\in\es\times\augSpace\ \rvert\  \ctg(\state,\augState,\TO)\leq \epsilon\}$ is precisely the set of starting states from which the optimal controller can guarantee achieving an ergodic metric of at most $\epsilon$.
    \end{remark}
    \par We solve the HJI equation using a deep neural network (DNN), following~\cite{bansal_deepreach_2021}, and we call the resulting controller \This. Implementation details are provided in \NEW{\apref{sec:Implementation}}.

    \subsection{Practical Considerations Regarding Scalability and Robustness}

    The HJI PDE \eqref{HJI Ergodic Problem} can only be solved numerically, which raises two concerns. Firstly, the solved value function is only an approximation, and errors in the approximation could invalidate the guarantees of Theorem \ref{thm:Guarantee}. This is a common issue for reachability methods; one common solution is to make the value function overly conservative, which is possible even when using neural networks~\cite{jiang2017usingneuralnetworkscompute}, but this compromises the accuracy of the value function. We do not provide such conservative guarantees here, but as numerical HJI solvers improve, any approximation errors and their consequences should diminish.

    Secondly, scalability is a major limitation for HJI solvers, and the extra dimensionality from the augmented state \augState\ limits the scope of ergodic problems that can be solved. The number of added dimensions is $\kmax^\expDim$, which scales polynomially with respect to the resolution \kmax\ of the basis functions and exponentially with respect to the dimensionality $\expDim$ of the exploration space $\expSpace$. The curse of dimensionality is common to all reachability applications, and developing new solvers with better scalability is an active area of research.
    
    Alternative formulations of reachability for ergodic search may be possible, but we do not anticipate that any such alternative formulation could avoid the issue of increased dimensionality. The ergodic metric depends on the history of the trajectory, so an optimal controller must have a way to respond to information about the history of the trajectory up to the current time. 
    Whether this information is stored in the augmented state variable $\augState$ (as in this paper, as well as~\cite{mathew_metrics_2011,de_la_torre_ergodic_2016,dong2023}) or some other way, it will necessarily constitute an additional input to the value function, thus increasing the dimensionality of the problem. 
    \NEW{Indeed, we show in \apref{sec:Derivation} that extended state variables emerge from the application of reachachability analysis to ergodic exploration, even when using an ergodic formulation without extended states (Problem \ref{def:Disturbed Ergodic}).}
    
\subsection{An Alternative Robust Ergodic Model-Predictive Controller} \label{sec:reMPC}
    \par We additionally develop a robust ergodic model-predictive controller, which we call reMPC. This controller uses first-order gradient descent to solve a discretized version of the minimax condition of the Bolza-form ergodic optimization problem~\eqref{Ergodic Bolza Form}:\footnote{
    Here, the dynamics are discretized using the explicit Euler method, but any numerical integration method could be used.}
    \begin{myproblem}[Discretized Bolza-form ergodic control with disturbance] 
    Given a start state $(\state\in\es,\augState\in\augSpace)$, for some timestep $\Delta t\in\RRR_{>0}$ and time horizon $\N\in\mathbb{N}$, solve for the optimal control $\ctrl^*\in\ctrlSpace^\N$ such that 
    \begin{align}
        \uDisc^*=\arg\min_{\uDisc\in\ctrlSpace^\N}\max_{\dDisc\in\distSpace^\N}\bigg\{
        \sum_{\dummytime=0}^{\N}{\running(\xDisc_\dummytime,\uDisc_\dummytime)\Delta \time}
        +\frac{1}{\N\Delta \time}\sumk{\Lambda_k(\zDisc_\N)_k^2}+\term(\xDisc_\N)\bigg\},
        \label{ReMPC Problem}
    \end{align}
    subject to
    \begin{subequations}
         \begin{align}
             \xDisc_\dummytime &= \xDisc_{\dummytime-1}+\dynReg(\xDisc_{\dummytime-1},\uDisc_{\dummytime-1},\dDisc_{\dummytime-1})\Delta t,\ \forall \dummytime \in [1,2,\dots,\N] \\
             (\zDisc_\dummytime)_k &= (\zDisc_{\dummytime-1})_k+(\basisK(\expMap(\xDisc_{\dummytime-1})-\phi_k)\Delta t,\ \forall \dummytime \in [1,2,\dots,\N], \forall k\in\kSpace\\
             (\xDisc_0,\zDisc_0)&=(\state,\augState).
         \end{align}
    \end{subequations}
    \end{myproblem}
    \par The complete reMPC algorithm is given as \aref{alg:reMPC}. The first-order gradient descent method (lines \ref{lst:line:grad_desc_start}-\ref{lst:line:grad_desc_end}) produces only local optima, so reMPC has no performance guarantees. However, given a long enough time horizon, if reMPC finds solutions close to the global optimum, then its performance should be similar to that of \This. 
    \vspace{-15pt}
    \begin{algorithm}[h]
        \KwIn{Starting state $(\state_0\in\es, \augState_0\in\augSpace)$}
        \KwIn{Dynamics $\dyn:\es\times\ctrlSpace\times\distSpace\to\mathcal{T}_\es$}
        \KwIn{Basis functions $\{\basisK:\expSpace\to\RRR,\forall k\in\kSpace\}$, information $\{\phi_k\in\RRR,\forall k\in\kSpace\}$}
        \KwIn{Running cost $\running:\es\times\ctrlSpace\to\Rpos$, terminal cost $\term:\es\to\Rpos$}
        \DontPrintSemicolon
        \SetKwBlock{Begin}{function}{end function}
        \Begin(\textsc{reMPC}{(}$\state_0\in\es, \augState_0\in\augSpace${)})
        {
            $\uDisc, \zDisc \gets \mathbf{0}$\;
            \While{$\mathrm{controller\ is\ running}$}
            {
                $\uDisc_{0,\dots,\N-2}\gets\uDisc_{1,\dots,\N-1}$,\ $\uDisc_{\N-1}\gets \mathbf{0}$\;
                $\dDisc_{0,\dots,\N-2}\gets\dDisc_{1,\dots,\N-1}$,\ $\dDisc_{\N-1}\gets \mathbf{0}$\;
                \For{$\niter\ \mathrm{iterations}$\label{lst:line:grad_desc_start}}
                {
                    $\uDisc \gets \uDisc - \lambda_\ctrl \cdot \left(\partial\text{\sc{PredRollout}}/\partial\uDisc\right)$\;
                    $\dDisc \gets \text{clip}\left(\dDisc + \lambda_\dist \cdot \left(\partial\text{\sc{PredRollout}}/\partial\dDisc\right), \distSpace\right)$ \Comment{Constrain $\dist\in\distSpace$}\label{lst:line:grad_desc_end}\;
                }
            Apply $\uDisc_0$
            }
        }
        \Begin($\textsc{PredRollout} {(}\state_0\in\es,\augState_0\in\augSpace,\uDisc\in\ctrlSpace^\N, \dDisc\in\distSpace^\N{)}$)
        {
            $\xDisc_0 \gets \state_0,\ \zDisc_0 \gets \augState_0$\;
            \For{$\dummytime \mathrm{\ in\ }[1, 2, \dots, \N]$ }
            {
                $\xDisc_{\dummytime} \gets \xDisc_{\dummytime-1}+\dynReg(\xDisc_{\dummytime-1}, \uDisc_{\dummytime-1}, \dDisc_{\dummytime-1})\Delta t$\;
                $(\zDisc_{\dummytime})_k \gets (\zDisc_{\dummytime-1})_k+(\basisK(\expMap(\xDisc_{\dummytime-1})-\phi_k)\Delta t,\  \forall k\in\kSpace$\;
            }
            \Return $\sum_{\dummytime=0}^{\N}{\running(\xDisc_\dummytime,\uDisc_\dummytime)\Delta \time}
        +\frac{1}{\N\Delta \time}\sumk{\Lambda_k(\zDisc_\N)_k^2}+\term(\xDisc_\N)$\;
        }
        \caption{Robust Ergodic Model-Predictive Control (reMPC)}\label{alg:reMPC}
    \end{algorithm}
    \vspace{-35pt}
\section{Results} \label{Results}

    \par In this section, we demonstrate the \This\ controller for the double-integrator system moving along one axis. 
    
    \REVISION{
    Due to the computational complexities of the value function, we found that training the value function was only possible for a one-second time horizon. 
    The results in this section were achieved by running the controller with a receding one-second time horizon, always inputting $\time\approx\TO$ into the controller.\footnote{\REVISION{Because $\time=\TO$ was on the boundary of the time interval over which the DNN was trained, a slightly larger time ($\time=0.02s$) in the interior of the time interval was found to provide better behavior.}} This modification violates the preconditions for the guarantees of Theorem \ref{thm:Guarantee} to apply, but if the HJI equation could be solved over a longer time horizon, the guarantees would apply over that time horizon.}

    We generate trajectories from \This\ and compare against reMPC (Section \ref{sec:reMPC}). We also compare against a standard model-predictive controller (MPC), a type of controller that does not explicitly account for disturbance but has often been shown to perform well under external disturbance~\cite{salman_multi-agent_2017,miller_ergodic_2016}. We implement a conventional MPC that minimizes the Bolza-form ergodic cost function \eqref{Ergodic Bolza Cost} without considering disturbance.     
    
    The goal for this section is to answer the following questions:
    \begin{enumerate}
        \item[Q1:] Can \This\ generate robust ergodic trajectories in the presence of a disturbance?
        \item[Q2:] How does the value function evolve over the course of a trajectory?
        \item[Q3:] How does the performance of \This\ compare to the performance of MPC and reMPC?
        \item[Q4:] Regarding its effects on controller performances, how does the worst-case disturbance generated by \This\ compare to other types of disturbance?
    \end{enumerate}
    \par The following subsections provide an overview of the problem setup and answers to the questions stated above\REVISION{; each subsection \textbf{A}$n$ in \ref{Answers} below answers the corresponding question \textbf{Q}$n$ above.}
    
\subsection{Problem Setup: One-Axis Double Integrator} \label{Setup}

    \par We use a single-axis double-integrator dynamical system to study the proposed method. The system has  state $\state=[\pos\ \vel]\tp$, control $\ctrl=\left[\ctrlel\right]$, disturbance $\dist=\left[\distel\right]$, and dynamics $\dot{\state}=[\vel, u+d]\tp$.
    Here, we subject the system to the control and disturbance constraints $\ctrl\in\left[-\ctrl_{\text{max}},\ctrl_{\text{max}}\right]\subset \mathbb{R},\ \dist\in\left[-\dist_{\text{max}},\dist_{\text{max}}\right]\subset \mathbb{R}$. For all results in this paper, we use $\dist_{\text{max}}=0.4\ctrl_{\text{max}}$; performance was found to decline significantly for larger disturbance magnitudes.
    
    \par We set $N=6$, leading to an 8-dimensional system: one dimension for time, two dimensions for the double integrator state, and five dimensions for the augmented state.\footnote{$N=6$ only necessitates five augmented state variables, because $z_{k=\textbf{0}}(t)=0$ for all $t$, so there is no need to compute and track this term.}
    Note that while the plant dynamics are linear, the dynamics of the augmented state variables \eqref{eq:dynAug} are nonlinear, so the augmented system as a whole is nonlinear.
    \par The ergodic problem setup was that of \eqref{HJI Ergodic Problem}, with running cost
    \begin{equation}   
    \running(\state, \augState,\ctrl)\defeq\runningWeight\sumk{\Lambda_k\augStateK^2}+\ctrl\tp R \ctrl + \barr(\state),
        \label{Running Cost}
    \end{equation}
    where \runningWeight\ is a nonnegative scalar weighting factor, $R\in\Re^{1\times 1}$ is a positive semi-definite matrix representing a weight on the control input, and $\barr(\regState)$ is a barrier function that penalizes the trajectory for leaving the search area bounds on position. We introduce the augmented state \augState\ into the running cost with the running ergodic cost term $\runningWeight\sumk{\Lambda_k\augStateK^2}$, which was found to improve convergence by helping the solver avoid a local minimum where $\ctg=0$ everywhere except close to $\time=\TF$. Other than the ergodic metric, there is no additional terminal cost $\term(\regState)$.
    
    \begin{figure*}[t!]
        \centering
        \includegraphics[width=\linewidth]{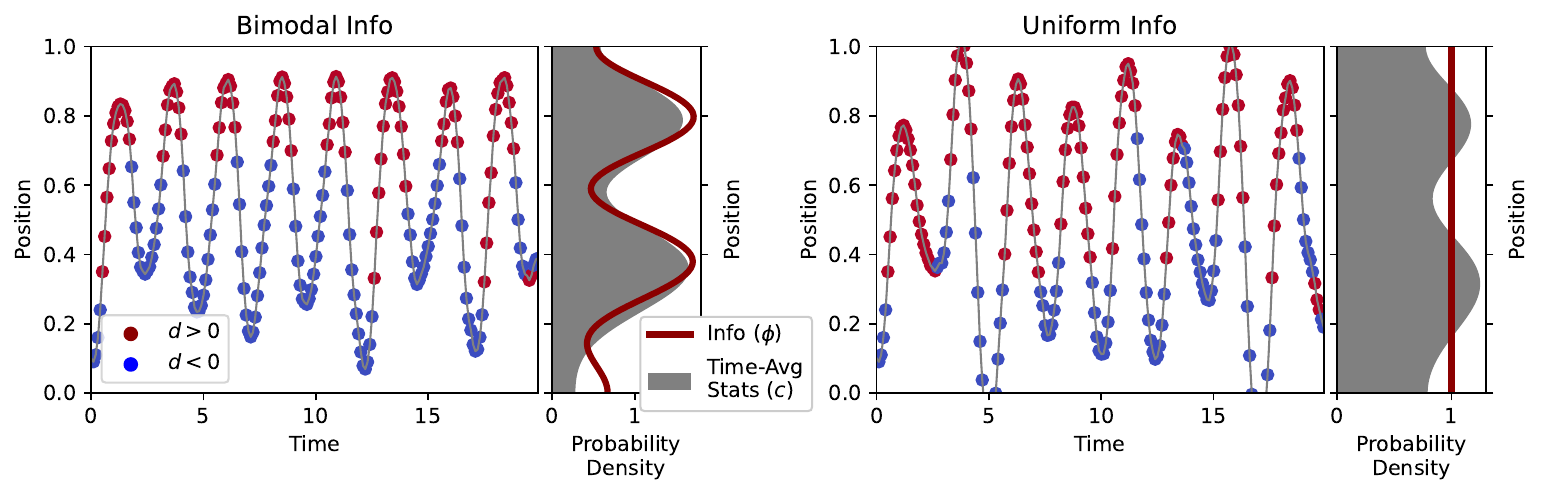}
        \caption{\textbf{Trajectories generated by \This\ in simulation, with worst-case disturbance, for bimodal and uniform information distributions.} For each pair of plots, the left plot shows the trajectory and the direction of the disturbance, and the right plot shows the coverage achieved by the trajectory. }
        \label{fig:Single Trajectory}
    \end{figure*}
\subsection{Simulated Results} \label{Answers}

    \subsubsection{A1: Robust ergodic trajectories.} 
    A robust controller must produce ergodic trajectories even with worst-case disturbance, and with the controller from \This, it is possible to calculate this worst-case disturbance \eqref{eq:range_disturbance}. In this particular problem setup, the control and disturbance are both affine, and they enter the dynamics in exactly the same way; 
    it follows that the worst-case disturbance is always in the opposite direction of the optimal control. Thus, the worst-case disturbance can be calculated as 
    \begin{equation}
        \dist^*=(-\mathrm{sign}(\ctrl^*))\dist_{\mathrm{max}}.
        \label{Noise}
    \end{equation}
    
    \begin{figure*}[t!]
        \centering
        \includegraphics[width=\linewidth]{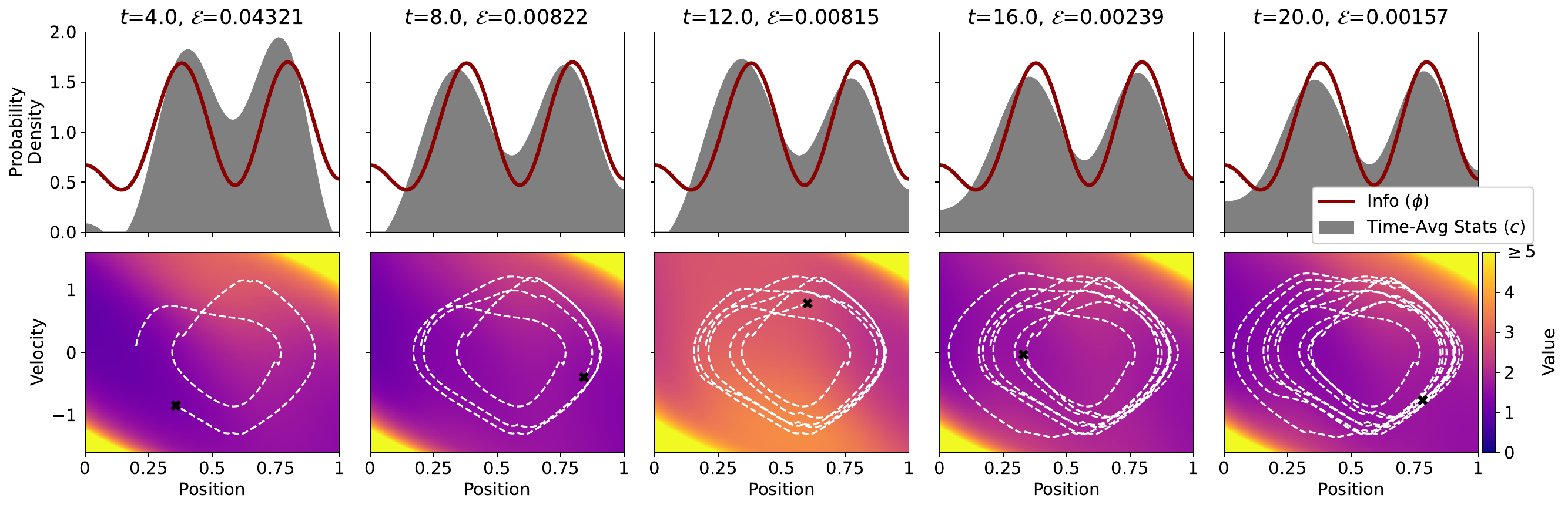}
        \caption{\textbf{Time-Evolution of a Trajectory and the Value Function.} Top row: the time-averaged spatial distribution and the information distribution. As time increases, the time-averaged trajectory statistics approach the information distribution, which is the goal of ergodic exploration. Bottom: evolution of the trajectory through phase space, superimposed on a cross-section of the value function with respect to position and velocity, where $z_k$ are fixed to their value at that instant in the trajectory. }
        \label{fig:Trajectory with Value}
    \end{figure*}
    \par We demonstrate that the controller can generate robust ergodic trajectories with worst-case disturbance, as shown in \figref{Single Trajectory}. More trajectory results are summarized in \figref{Comparisons}, discussed below.

    \subsubsection{A2: Value function evolution.} \figref{Trajectory with Value} shows the evolution of a trajectory, as well as a cross-section of the value function at times during the trajectory. Qualitatively, properties of the value function can be related to the problem setup. Recall that the controller seeks states where the value function is lowest. At the top right and lower left corners, the value function is high, indicating that the controller should avoid these states; this is logical, because the agent is about to leave the exploration bounds and incur a boundary cost. 
    Elsewhere, the value function is lowest in states where the agent is at or moving toward positions it has not adequately explored, and highest in states where the agent is at or moving toward positions it has already over-explored.

    \subsubsection{A3: Comparison of \This, MPC and reMPC.} The performance of the three controllers for two information distributions and four disturbance types is summarized in \figref{Comparisons}. There is considerable variation in the performances within each trial, due in part to the limitations of the controllers' short time horizons, but some trends emerge.

        \begin{figure*}[t!]
        \centering
        \includegraphics[width=\linewidth]{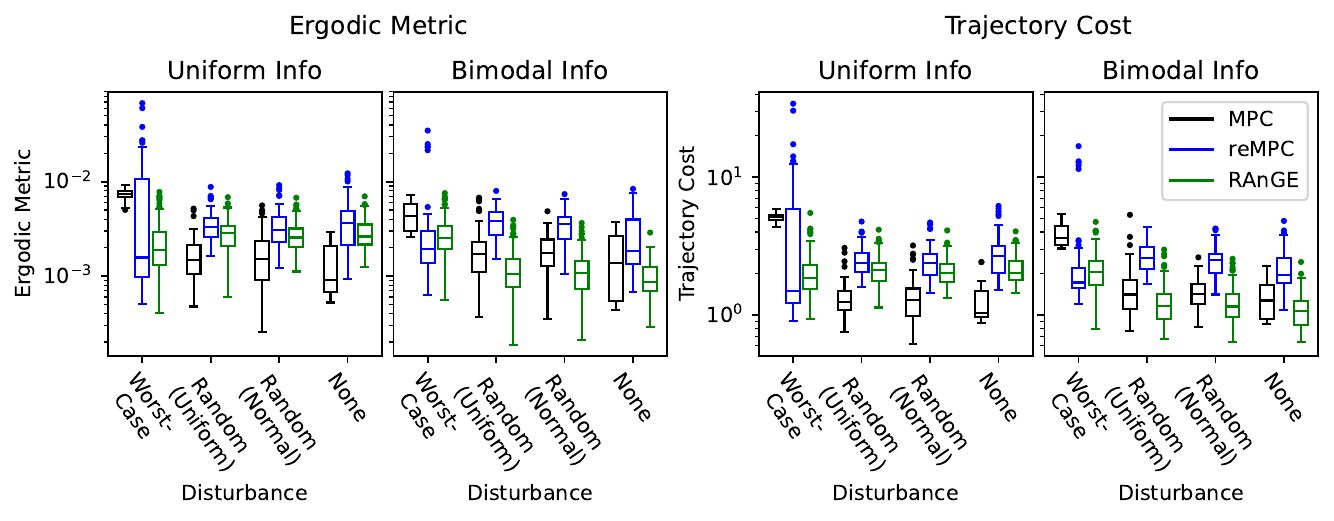}
        \caption{\textbf{Comparison of MPC, reMPC, and \This}.
        The performance measures are the ergodic metric \eqref{Equivalent Metrics} and cost function \eqref{Ergodic Bolza Cost}. Note that the ergodic metric is not normalized; see \figref{Trajectory with Value} for an intuition of the scale of \ergmet.
        The MPC and reMPC results aggregate 64 runs with varying initial conditions; the \This\ results aggregate from 64 runs each with two controllers trained with different random seeds. 
        The worst-case disturbance came from \This\ according to \eqref{eq:range_disturbance}, the uniform disturbance was sampled from $[-\dist_{\mathrm{max}},\dist_{\mathrm{max}}]$, and the normal disturbance was sampled from a Gaussian with $\mu=0,\sigma=\dist_{\mathrm{max}}/2$. All trials lasted for 20s, and all controllers had a 1s time horizon, so the running cost was divided by 20 when computing the trajectory cost $\cost$, to preserve the relative weighting of the running and terminal costs.}
        
        \label{fig:Comparisons}
    \end{figure*}
    \par As expected, the worst performances by \This\ are better than the worst performances by either of the other two methods. This is the primary theoretical advantage of \This: it should have the highest floor for performance of any controller. By contrast, the best performance by the other controllers rival or exceed the performance of \This, particularly with non-worst-case noise. This is consistent with the formulation of \This, which is designed to optimize its worst-case performance, not its best-case performance. In situations where average performance with non-worst-case disturbance is most important, the MPC could be a better choice.
    \par Interestingly, the reMPC controller performs its best---both absolutely and relative to the other methods---with worst-case disturbance. This is because the reMPC controller optimizes its controls under the assumption of a near-worst-case disturbance, so it is not optimized for non-worst-case disturbances. Consequently, reMPC outperforms MPC with worst-case disturbance but underperforms it otherwise.
    \vspace{-15pt}\subsubsection{A4: Comparison of Disturbances} Just as the \This\ controller optimizes its worst-case performance, the \This-generated disturbance is optimally designed to impede best-case performance by any controller. Across both information distributions, the best performance by any controller with worst-case disturbance is worse than the best performance by any controller with each of the other types of noise.
    \par The effects of the two types of random disturbance on controller performance are very similar to the effect of no disturbance at all. This is unsurprising, because the random disturbances average to zero and thus tend to cancel out.

    \begin{figure}[t]
         \centering 
         \includegraphics[width=\linewidth]{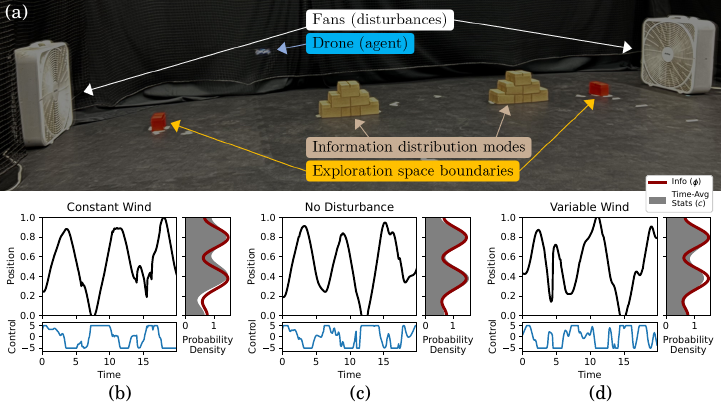}
    	\caption{\textbf{Experimental setup and results.} (a) We used a Bitcraze Crazyflie 2.1 quadrotor drone, and two Lighthouse cameras (not pictured) to assist with position and velocity measurement~\cite{crazyswarm}. The information distribution was a hard-coded bimodal distribution, and the blocks show the modes and bounds for visualization purposes only. Bottom row: trajectories recorded from the drone, with the fans (b) turned on at constant low speed, (c) turned off, and (d) turned on at varying speeds from low to high over the course of the trajectory.}
    	\label{fig:Experiment}
    \end{figure}

    \subsection{Experimental Results}
    \par We tested the controller on a Bitcraze Crazyflie 2.1 quadrotor drone~\cite{crazyswarm} with fans for disturbance; the experimental setup is shown in \figref{Experiment}a. The drone was tested with no disturbance; with the fans at constant low speed; and with the fans varying between low, medium, and high speeds. Trajectories from these tests are shown in subfigures b-d of \figref{Experiment}. We were unable to implement a worst-case disturbance in the experiment because the worst-case disturbance must react instantaneously to the trajectory of the drone.
    
    \par In all cases, the drone executed ergodic trajectories, as shown with the reconstructed time-averaged statistics in \figref{Experiment}. This demonstrates that the \This\ controller is robust to disturbance and can be implemented on physical systems. The trajectory was most ergodic in the case of no disturbance, which is consistent with the results of \figref{Comparisons}.

\section{Discussion} \label{Conclusion}

    \par In conclusion, we showed that it is possible to guarantee coverage using reachability analysis on ergodic search problems despite dynamic disturbances. We demonstrated a Bolza-form formulation for ergodic exploration using an augmented state that allows us to derive the ergodic HJI conditions of optimality. To compute solutions to the ergodic HJI conditions, we leveraged recent deep learning methods to approximate the value function. The controller was demonstrated to provide robust ergodic trajectories in the presence of dynamic disturbances, both in simulation and in experiments.
    \par One limitation of this work is that due to the value function complexity, the HJI equation could not be solved for more complex problems (e.g., time horizons much greater than one second and  exploration spaces with multiple dimensions). 
    This limitation is a result of numerical implementation and not the fundamental theory. With advances in HJI solving techniques, future work could extend these results to longer time horizons, enabling longer guarantees on ergodic coverage trajectories. Furthermore, we seek to extend this work to higher, multidimensional search spaces and introduce dynamic obstacles in the formulation. This paper provides the theoretical foundation for performing reachability analysis on ergodic exploration, and improved solving techniques will broaden its applicability.
\begin{appendix}

\section{Implementation for Solving the HJI Equation} \label{sec:Implementation}
    \par Solving the Hamilton-Jacobi-Isaacs problem \eqref{HJI Ergodic Problem} for the value function \ctg\ is not feasible analytically, but it can be done numerically. Following \cite{bansal_deepreach_2021}, we employ a deep neural network (DNN) to compute a numerical solution for \ctg, with the sine function as the activation function. The sine activation function provides smooth gradients of the output of the neural network \cite{sitzmann_implicit_2020}, which is necessary for computing the controller and disturbance \eqref{eq:range_ctrl_dist}. The input to the DNN is the full state $(\regState, \augState,\time)$, and the output is trained to be the scalar value function \ctg. 
    
    \par The training loss \Loss\ is defined as $\Loss(\params) =\BoundLoss(\params)+\lossLambda \DiffLoss(\params)$, where
    \begin{subequations}
    \begin{align}
            \BoundLoss(\params)&\defeq\frac{1}{|S_{\TF}|} \sum_{(\regState, \augState,\TF)\in S_{\TF}}{\vmag{\DNN(\regState, \augState,\TF)-\frac{1}{\TF-\TO} \sumk{\Lambda_k \augStateK^2}-\term(\regState)}},\\
            \DiffLoss(\params)&\defeq\frac{1}{|S|}\sum_{(\regState,\augState,\time)\in S}{\Biggr{\lVert}\Partial{\DNN}{\time}+\min_{\ctrl\in\ctrlSpace}\max_{\dist\in\distSpace}
            \biggr{\{}
                \hat{\pHam}(\regState,\augState,\time, \ctrl,\dist)   
            \biggl{\}}\Biggl{\rVert}},
        \label{Loss}
    \end{align}
    \end{subequations}
    where \DNN\ is the output of the DNN with parameters \params, $S$ is a sampled set of state data points of the form $(\regState,\augState, \time)$, and $S_\TF\subseteq S$ is the subset of $S$ of all state data points with $\time=\TF$, and $\hat{\pHam}$ is defined in \eqref{Ergodic HJI Hamiltonian}. 
    
    \par In the loss function, \BoundLoss\ is the loss associated with the terminal boundary condition \eqref{Ergodic HJI Terminal}, so it is only considered for state data points where $\time=\TF$. By contrast, \DiffLoss\ is the loss associated with the HJI differential equation \eqref{Ergodic HJI Diff}, so it is considered for all state points sampled. The tuning parameter \lossLambda\ is used to adjust the weighting of the terminal and differential loss components.

    \par The training of the DNN was performed as in~\cite{bansal_deepreach_2021}, such that the terminal condition \eqref{Ergodic HJI Terminal} was learned first, and then the solution was learned expanding backwards from $\time=\TF$ to $\time=\TO$, following the differential equation \eqref{Ergodic HJI Diff}. 
    The DNN was trained for a total of $1.5\times10^6$ iterations, which took an average of 247 minutes on an NVIDIA GeForce RTX 3080 GPU.
    

\section{Proofs}
\subsection{Proof of Lemma \ref{metric_lemma}} \label{metric_proof}
\begin{proof}
Integrating the differential equation for $\augStateK$ \eqref{eq:dynAug}, recognizing that $\phi_k$ is constant, produces
\begin{equation}
    \augStateK(\Traj,\TF)=
    \int_{\TO}^{\TF}{\basisK(\expMap\circ\regState(\time))d\time}-(\TF-\TO)\phi_k.
\end{equation}
The definition of $c_k$ \eqref{Fourier Coefficients} can be used to simplify to 
 \begin{equation}
 \augStateK(\traj,\TF)=
    (\TF-\TO)\left(c_k(\Traj,\TF)-\phi_k\right).
 \end{equation}
 Finally, substituting this into \eqref{Equivalent Metrics} yields
 \begin{equation}
 \frac{1}{\TF-\TO} \sumk{\Lambda_k \augStateK^2}
    = \sum_{k\in\kSpace}{\Lambda_k\left(c_k(\Traj,\TF)-\phi_k\right)^2}.
    \label{End of Proof}
 \end{equation}
This completes the proof. \qed
 \end{proof}
\subsection{Proof of Theorem \ref{thm:Guarantee}} \label{pr:Guarantee}
\begin{proof}
    For an alternative proof, consult~\cite{evans_differential_1983}. In \eqref{Cost-to-go}, repeated here, we define the value function as the optimal cost of the remainder of the trajectory, given by
    \begin{align}
        \ctg(\state,\augState,\Ti)&=\min_{\UPol}\max_{\DPol}\left\{\cost(\state,\Ti, \UPol, \DPol\right)\}.
    \end{align}
    Using the ergodic cost function \eqref{Ergodic Bolza Cost}, this becomes
    \begin{align}
        \ctg(\state,\augState,\Ti)=&\min_{\UPol}\max_{\DPol}\bigg\{
        \int_{\Ti}^{\TF}{\running(\state(\dummytime),\ctrl(\dummytime))d\dummytime} \nonumber\\
        &\phantom{\min_{\UPol}\max_{\DPol}\bigg\{}+\frac{1}{\TF-\TO}\sumk{\Lambda_k\augStateK(\TF)^2}+\term(\state(\TF))\bigg\}.
        \label{eq:Expanded V}
    \end{align}
    Note that in this and following equations, $\state(\time),\augState(\time),\ctrl(\time)$ are governed by the constraints \eqref{eq:first_traj_cond}-\eqref{eq:last_traj_cond}.
        Substituting $\Ti=\TF$ into the definition of $\ctg$ \eqref{eq:Expanded V} directly yields the HJI boundary condition
        \begin{equation}
        \ctg(\state,\augState,\TF)=\frac{1}{\TF-\TO}\sumk{\Lambda_k\augStateK(\TF)^2}+\term(\state(\TF)).
        \end{equation}
        
    \noindent Next, we define $\time\in(\Ti,\TF)$, then split the integral in \eqref{eq:Expanded V}, yielding
        
    \begin{align}
        \ctg(\state,\augState,\Ti)=&\min_{\UPol}\max_{\DPol}\bigg\{
        \int_{\Ti}^{\time}{\running(\state(\dummytime),\ctrl(\dummytime))d\dummytime}+
        \int_{\time}^{\TF}{\running(\state(\dummytime),\ctrl(\dummytime))d\dummytime} \nonumber \\
        &\phantom{\min_{\UPol}\max_{\DPol}\bigg\{}+\frac{1}{\TF-\TO}\sumk{\Lambda_k\augStateK(\TF)^2}+\term(\state(\TF))\bigg\}\\
        =&\min_{\UPol}\max_{\DPol}\bigg\{
        \int_{\Ti}^{\time}{\running(\state(\dummytime),\ctrl(\dummytime))d\dummytime}+
        \ctg(\state(\time),\augState(\time),\time)\bigg\}
        \label{eq:HJB Condition}
    \end{align}
    Equation \eqref{eq:HJB Condition} is the Isaacs recursive principle of optimality (c.f. Eq. 2.7 of \cite{galperin2008}).
    Taking the derivative with respect to $\time$ yields
    \begin{align}
        0=&\min_{\uPol(\time)}\max_{\dPol(\time)}\left\{
        \running(\state(\time),\ctrl(\time)) +\Partial{\ctg}{\state}\frac{d\traj(\time)}{d\time} +\Partial{\ctg}{\augState}\frac{d\augState(\time)}{d\time}+ \Partial{\ctg}{\time}\right\} \\
        -\Partial{\ctg}{\time}=&\min_{\uPol(\time)}\max_{\dPol(\time)}\left\{
        \running(\state(\time),\ctrl(\time)) +\Partial{\ctg}{\state}\frac{d\traj(\time)}{d\time} +\Partial{\ctg}{\augState}\frac{d\augState(\time)}{d\time}\right\}
    \end{align}
    Since all terms are evaluated at time $\time$, we can replace $\state(\time)$ with $\state$, and likewise for $\augState,\ctrl,\dist$.
    We can now substitute the equations for $\frac{d\state}{d\time}$ \eqref{Disturbance Dynamics} and $\frac{d\augState}{d\time}$ \eqref{eq:dynAug}, yielding the HJI differential condition
        \begin{align}
                -\Partial{\V(\state,\augState,\time)}{\time}&=\min_{\ctrl\in\ctrlSpace}\max_{\dist\in\distSpace}\bigg\{
                \running(\regState, \ctrl)+
                \Partial{\V(\state,\augState,\time)}{\regState}
                \dynReg(\regState,\ctrl,\dist)\nonumber \\
                &\phantom{\min_{\ctrl\in\ctrlSpace}\max_{\dist\in\distSpace}\bigg\{abc}+\sumk{\Partial{\V(\state,\augState,\time)}{\augStateK}\left(\basisK(\expMap\circ\state)-\phi_k
                \right)} 
                \bigg\}.
                \label{eq:Ergodic HJI Diff proof}
        \end{align}
        By inspection, it is clear that for any $(\state,\augState,\time)$, the optimal control \ctrl\ and disturbance \dist\ are precisely those that satisfy the minimax condition in \eqref{eq:Ergodic HJI Diff proof}.
        \qed
    \end{proof}

\section{Derivation of the Ergodic HJI Equation }\label{sec:Derivation}
This section presents an alternative derivation of the HJI equation for ergodic exploration, presented above as \eqref{HJI Ergodic Problem}. Notably, although we begin from the canonical formulation of the ergodic metric, which does not include extended state variables, the resulting HJI equation will nevertheless contain extended state variables.
\subsection{Problem Setup}
We start with the simple cost function
\begin{align}
	\cost(\Traj,\UPol) \defeq &\Run{\TO}{\TF}+\term(\state(\TF))+
    \ergmet(\Traj,\TF)\\
    =& \Run{\TO}{\TF}+\term(\state(\TF))\nonumber \\
    &+\sum_k {\LamK\left(\frac{1}{\TF-\TO}\int_\TO^\TF\basisK(\expMap\circ\traj(\time))d\time-\phi_k \right)^2}
\end{align}
First, we bring the $\phi_k$ into the integral of the ergodic metric, yielding
\begin{align}
    \cost(\Traj,\UPol)=&\Run{\TO}{\TF}+\term(\state(\TF)) \nonumber \\
    &+\sum_k {\LamK\left(\frac{1}{\TF-\TO}\int_\TO^\TF\left(\basisK(\expMap\circ\traj(\time))-\phi_k\right)d\time \right)^2}
    \label{eq:Unexpanded Cost}
\end{align}
For notational ease, we define $\epK(\time)\defeq \basisK(\expMap\circ\traj(\time))-\phi_k$.
The ergodic metric is the sum of a single squared term for each value of $k$, so it can be rewritten as 
\begin{equation}
\ergmet(\traj, \TF) = \sum_{k}\left(\frac{\Lambda_k}{(\TF-\TO)^2}\zIntp{\TO}{\TF}^2\right),
\label{eq:MetSimple}
\end{equation}
\subsection{Solving for the HJI Equation}
Define times $\Ti,\Tii$ such that $\TO<\Ti<\Tii<\TF$.
We expand the expression for the cost function \eqref{eq:Unexpanded Cost} by substituting in the definition of the ergodic metric \eqref{eq:MetSimple} and recognizing that $\zInt{\TO}{\TF}=\zInt{\TO}{\Ti}+\zInt{\Ti}{\Tii}+\zInt{\Tii}{\TF}$, yielding

\begin{align}
    \cost(\Traj,\UPol) 
    =&
    \Run{\TO}{\TF}+\term(\state(\TF))\nonumber\\
    &+\sum_{k\in\kSpace}\Bigg(\frac{\Lambda_k}{(\TF-\TO)^2}\bigg(\zInt{\TO}{\Ti}+\zInt{\Ti}{\Tii}+\zInt{\Tii}{\TF}\bigg)^2\Bigg).
\end{align}
Expanding yields
\begin{align}
    \cost(\Traj,\UPol)=&\Run{\TO}{\Ti}+\Run{\Ti}{\Tii}+\Run{\Tii}{\TF}+\term(\state(\TF))\nonumber \\ 
    &+\sum_{k\in\kSpace}\Bigg(\frac{\Lambda_k}{(\TF-\TO)^2}\Bigg[ \zIntp{\TO}{\Ti}^2 
    +\zIntp{\Ti}{\Tii}^2  \nonumber\\
    &\hspace{1in}+2\zIntp{\TO}{\Ti}\zIntp{\Ti}{\Tii}  \nonumber\\
    &\hspace{1in}+2\left[\zInt{\TO}{\Ti}+\zInt{\Ti}{\Tii}\right]\zIntp{\Tii}{\TF}  \nonumber\\
    &\hspace{1in}+\zIntp{\Tii}{\TF}^2\Bigg]\Bigg)\Bigg\}.
    \label{Expanded}
\end{align}
We define the cost-to-go $\ctg(\subtraj{\TO}{\dummytime},\dummytime)$ as the optimal value of the remaining running cost and the ergodic metric of entire trajectory, where the portion of the trajectory prior to $\time=\dummytime$ is held fixed. 
We use the subscript ${\boxed{\cdot}}_{[\time_a, \time_b]}$ to denote the portion of the trajectory, control, or disturbance function limited to the time range $t\in[\time_a, \time_b]$. 
Applying this definition of $\ctg$ to \eqref{Expanded} produces
\begin{align}
    \label{eq:VT2}
    \ctg(\subtraj{\TO}{\Tii},\Tii)=& \zIntp{\TO}{\Tii}^2\nonumber \\
    &+\umintn{\Tii}{\TF}\Bigg\{\Run{\Tii}{\TF}+\term(\state(\TF))\nonumber \\ 
    &\hspace{0.5in}+\sum_{k\in\kSpace}\Bigg(\frac{\Lambda_k}{(\TF-\TO)^2}\Bigg[
    2\zIntp{\TO}{\Tii}\zIntp{\Tii}{\TF}  \nonumber\\
    &\hspace{1.5in}+\zIntp{\Tii}{\TF}^2\Bigg]\Bigg)\Bigg\}
\end{align}
and
\begin{align}
    \label{eq:Fully expanded Value}
    \ctg(\subtraj{\TO}{\Ti},\Ti)=& \zIntp{\TO}{\Ti}^2\nonumber \\
    &+\umintn{\Ti}{\TF}\Bigg\{\Run{\Ti}{\Tii}+\Run{\Tii}{\TF}+\term(\state(\TF))\nonumber \\ 
    &\hspace{0.5in}+\sum_{k\in\kSpace}\Bigg(\frac{\Lambda_k}{(\TF-\TO)^2}\Bigg[\zIntp{\Ti}{\Tii}^2  \nonumber\\
    &\hspace{1.25in}+2\zIntp{\TO}{\Ti}\zIntp{\Ti}{\Tii}  \nonumber\\
    &\hspace{1.25in}+2\left[\zInt{\TO}{\Ti}+\zInt{\Ti}{\Tii}\right]\zIntp{\Tii}{\TF}  \nonumber\\
    &\hspace{1.25in}+\zIntp{\Tii}{\TF}^2\Bigg]\Bigg)\Bigg\}\\
    =& \umintn{\Ti}{\Tii}\Bigg\{
    \Run{\Ti}{\Tii}+
    \zIntp{\TO}{\Tii}^2\nonumber \\
    &\hspace{0.25in}+\umintn{\Tii}{\TF}\Bigg\{\Run{\Tii}{\TF}+\term(\state(\TF))\nonumber \\ 
    &\hspace{0.5in}+\sum_{k\in\kSpace}\Bigg(\frac{\Lambda_k}{(\TF-\TO)^2}\Bigg[
    2\zIntp{\TO}{\Tii}\zIntp{\Tii}{\TF}  \nonumber\\
    &\hspace{1.5in}+\zIntp{\Tii}{\TF}^2\Bigg]\Bigg)\Bigg\}\Bigg\}\\
    =&\umintn{\Ti}{\Tii}\Big\{
    \Run{\Ti}{\Tii}+\ctg(\subtraj{\TO}{\Tii}, \Tii)
    \Big\}.
    \label{eq:Ergodic Bellman}
\end{align}

The final result \eqref{eq:Ergodic Bellman} is the recursive Isaacs principle of optimality for a problem with running cost. Note that aside from the different notations for the arguments of $\ctg$, this is equivalent to \eqref{eq:HJB Condition}, from the proof of Theorem \ref{thm:Guarantee} above.

As written in \eqref{eq:Fully expanded Value}, the cost-to-go $\ctg$ at time $\dummytime$ depends on the entire trajectory up to that time, $\subtraj{\TO}{\dummytime}$. However, inspection shows that $\subtraj{\TO}{\dummytime}$ is only used in two ways:
\begin{itemize}
    \item The entire history $\subtraj{\TO}{\dummytime}$ is used in the calculation of $\zInt{\TO}{\dummytime}, \forall k\in\kSpace$.
    \item The state $\trajf{\dummytime}$ is necessary for rolling out the dynamics at times greater than $\dummytime$.
\end{itemize}
Thus, $\ctg$ at time $\dummytime$ can be written as $\ctg(\trajf{\dummytime},\zIntSet{\TO}{\dummytime},\dummytime)$.
Taking the limit of \eqref{eq:Ergodic Bellman} as $\Delta t \defeq \Tii-\Ti\to 0$ results in
\begin{align}
    \ctg&\left(\trajf{\Ti},\zIntSet{\TO}{\Ti},\Ti\right) \\
    &=\min_{\ctrl\in\ctrlSpace}\max_{\dist\in\distSpace}\Bigg\{\running(\trajf{\Ti},\uPol(\Ti)) \Delta t \nonumber \\ 
&\hspace{0.25in}+\ctg\left(\trajf{\Ti}+f(\state,\ctrl,\dist)\Delta t,\left\{\zInt{\TO}{\Ti}+\epK(\Ti)\Delta t,\forall k\in\kSpace\right\},\Ti+\Delta t\right)\Bigg\}. \nonumber
\end{align}
We use the Taylor expansion of the value function, resulting in
\begin{align}
    \ctg&\left(\trajf{\Ti},\zIntSet{\TO}{\Ti},\Ti\right) \\
    &=\min_{\ctrl\in\ctrlSpace}\max_{\dist\in\distSpace}\Bigg\{\running(\trajf{\Ti},\uPol(\Ti)) \Delta t 
+\ctg\left(\trajf{\Ti},\zIntSet{\TO}{\Ti},\Ti\right) \nonumber \\
&\hspace{0.25in}+\Partial{\ctg}{\state}\dynReg(\state,\ctrl,\dist)\Delta \time
    +\sumk{\Partial{\ctg}{\left(\zInt{\TO}{\time}\right)}\epK}\Delta \time+\Partial{\ctg}{\time}\Delta \time + o(\Delta \time^2)
\Bigg\}, \nonumber
\end{align}
where $o(\Delta \time^2)$ is a term that approaches 0 much faster than $\Delta \time$ as $\Delta \time \to 0$. Dividing by $\Delta \time$ and then taking the limit as $\Delta t \to 0$ yields
\begin{equation}
    0=\min_{\ctrl\in\ctrlSpace}\max_{\dist\in\distSpace}\Bigg\{\running(\state,\ctrl) 
    +\Partial{\ctg}{\state}\dynReg(\state,\ctrl,\dist)
    +\sumk{\Partial{\ctg}{\left(\zInt{\TO}{\time}\right)}\epK}+\Partial{\ctg}{\time}\Bigg\}.
\end{equation}
We recognize now that $\zInt{\TO}{\dummytime}$ is precisely equivalent to $\augStateK(\dummytime)$. We refrained from making this substitution earlier in order to emphasize that this derivation does not depend on any properties of $\augStateK$ proved elsewhere in the paper. Making this substitution, expanding $\epK$, and rearranging produces the differential condition
\begin{align}
    -\Partial{\ctg}{\time}=\min_{\ctrl\in\ctrlSpace}\max_{\dist\in\distSpace}\Bigg\{&\running(\state,\ctrl)+\Partial{\ctg}{\state}\dynReg(\state,\ctrl,\dist) \nonumber \\
    &
    +\sumk{\Partial{\ctg}{\augStateK}\left(\basisK(\expMap(\state))-\phi_k\right)}\Bigg\}.
    \label{eq:Running HJI Differential}
\end{align}
Finally, evaluating the non-recursive formulation of the value function \eqref{eq:VT2} at $\time=\TF$ yields the terminal condition
\begin{equation}
    \ctg(\state,\augState,\TF)=\sumk{\frac{\Lambda_k}{(\TF-\TO)^2}\augStateK^2}+\term(\state),\ \forall \state\in\es, \forall \augState\in\augSpace
    \label{eq:Running HJI Terminal}
\end{equation}
Taken together, \eqref{eq:Running HJI Differential} and \eqref{eq:Running HJI Terminal} constitute a PDE that defines the value function, and they are identical to the HJI PDE \eqref{HJI Ergodic Problem} derived from the extended-state formulation of the ergodic control problem (Problem \ref{prob:Ergodic Bolza}).

\end{appendix}

\bibliography{bibliography}
\bibliographystyle{ieeetr}

\end{document}